\documentclass{article}
\usepackage{spconf,amsmath,graphicx}
\usepackage{epsfig,epstopdf,subfigure,slashbox}
\usepackage{amsfonts}
\usepackage{datetime,algorithm,algorithmic,multirow}
\usepackage{xcolor,color}
\usepackage{cite}
\usepackage{amsthm}
\DeclareMathOperator*{\argmax}{argmax}

\title{STUDY ON FEATURE SUBSPACE OF ARCHETYPAL EMOTIONS FOR SPEECH EMOTION RECOGNITION}
\name{Xi Ma$^{1,3}$, Zhiyong Wu$^{1,2,3}$, Jia Jia$^{1,3}$, Mingxing Xu$^{1,3}$, Helen Meng$^{3}$, Lianhong Cai$^{1,3}$}

\address{$^1$Tsinghua-CUHK Joint Research Center for Media Sciences, Technologies and Systems,\\
Graduate School at Shenzhen, Tsinghua University, Shenzhen 518055, China\\
  $^2$Department of Systems Engineering and Engineering Management,\\
  The Chinese University of Hong Kong, Shatin, N.T., Hong Kong SAR, China\\
  $^3$Tsinghua National Laboratory for Information Science and Technology (TNList),\\
  Department of Computer Science and Technology, Tsinghua University, Beijing 100084, China\\
  {\small \tt max15@mails.tsinghua.edu.cn}, {\small \tt \{zywu,hmmeng\}@se.cuhk.edu.hk}, {\small \tt \{jjia,xumx,clh-dcs\}@tsinghua.edu.cn}
}
\begin{document}
%
\maketitle
\begin{abstract}
Feature subspace selection is an important part in speech emotion recognition. Most of the studies are devoted to finding a feature subspace for representing all emotions. However, some studies have indicated that the features associated with different emotions are not exactly the same. Hence, traditional methods may fail to distinguish some of the emotions with just one global feature subspace. In this work, we propose a new divide and conquer idea to solve the problem. First, the feature subspaces are constructed for all the combinations of every two different emotions (emotion-pair). Bi-classifiers are then trained on these feature subspaces respectively. The final emotion recognition result is derived by the voting and competition method. Experimental results demonstrate that the proposed method can get better results than the traditional multi-classification method.
\end{abstract}
\section{Introduction}
\label{sec:intro}

Emotion recognition plays an important role in many applications, especially in human-computer interaction systems that are increasingly common today. As one of the main communication media between human beings, voice has received widespread attention from researchers~\cite{ayadi2011}. Speech contains a wealth of emotional information, how to extract such information from the original speech signal is of great importance for speech emotion recognition.

As an important part of speech emotion recognition, the selection of feature subspace has attracted lot of research interests. Existing researches on feature subspace selection can be divided into three categories, including the artificial selection of emotion related features, the automatic feature selection algorithms to select feature subset from a large set of numerous feature candidates, and the transformation method to map the original feature space to the new one in favor of emotion recognition. Most of these researches are devoted to finding a common and global feature subspace that can represent all kinds of emotions. However, studies have already indicated that the features associated with different emotions are not exactly the same. In other words, if we can divide the whole emotions space into several subspaces and find the features that are most distinguishable for each subspace separately, the emotion recognition performance on the whole space might be boosted. Motivated by this, we propose a divide and conquer idea for emotion recognition by leveraging feature subspaces. The feature subspaces are first constructed for every two different emotions (i.e. emotion-pair); bi-classifier are then used to distinguish the emotions for each emotion-pair from the feature subspace; and the final emotion recognition result is derived by voting and competition method.

The reset of the paper is organized as follows. Section~\ref{sec:related_work} summarizes previous related work on feature selection. Our proposed method is then detailed in Section~\ref{sec:method}. Experiments and results are presented in Section~\ref{sec:expeiment}. Section~\ref{sec:conclusion} concludes the paper.

\section{Related Work}
\label{sec:related_work}

As a common issue for many classification problems~\cite{dash1997}, feature selection aims to pick a subset of features that are most relevant to the target concept~\cite{guyon2003} or to reduce the dimension of features for reducing computational time as well as improving the performance~\cite{belhumeur1997}. There have been many studies on feature selection for speech emotion recognition. In~\cite{busso2009, cowie2001, vayrynen2013}, prosody-based acoustic features, including pitch-related, energy-related and timing features have been widely used for recognizing speech emotion. Spectral-based acoustic features also play an important role in emotion recognition, such as Linear Predictor Coefficients (LPC)~\cite{rabiner1978}, Linear Predictor Cepstral Coefficients (LPCC)~\cite{atal1974} and Mel-frequency Cepstral Coefficients (MFCC)~\cite{davis1980}. In~\cite{gobl2003}, voice quality features have also been shown to be related to emotions.

Besides manual selection, there have also many automatic feature selection algorithms been proposed. For example, Sequential Floating Forward Selection (SFFS)~\cite{ververidis2008} is an iterative method that can find a subset of features near to the optimal one. Some evolutionary algorithms such as Genetic Algorithm (GA)~\cite{ferri1993} are often used in feature selection. Feature space transformation is another type of method, including Principal Component Analysis (PCA)~\cite{belhumeur1997} and Neural Network (NN)~\cite{yu2013}.

To describe emotions, some studies have used the psychological dimensional space such as the 2-dimensional arousal-valence model and the 3-dimensional valence-activation-dominance model~\cite{schlosberg1954}. Besides, discrete emotional labels, the so-called archetypal emotions~\cite{ortony1990}, are common used in speech emotion recognition. Different archetypal emotions are located at different locations in the dimensional space. \cite{lugger2008} has proposed a hierarchical approach to classify the speech emotions with the dimensional model. However, the selection of emotions at different stages is too subjective, and the used feature sets may not have a good match to the psychological emotional model.

\section{Method}
\label{sec:method}

Our study is based on archetypal emotions. The emotion-pair is composed of two different kinds of archetypal emotion, like Anger and Happiness. For all possible combinations of archetypal emotion-pairs, the bi-classification and voting method is used to distinguish every emotion-pairs and to derive the final emotion recognition result, As shown in Figure~\ref{fig:flow}, the whole method involves four steps: feature extraction, feature subspace selection, emotion classification and voting decision.

\begin{figure}[htb]
  \centering
  \centerline{\includegraphics[width=9cm]{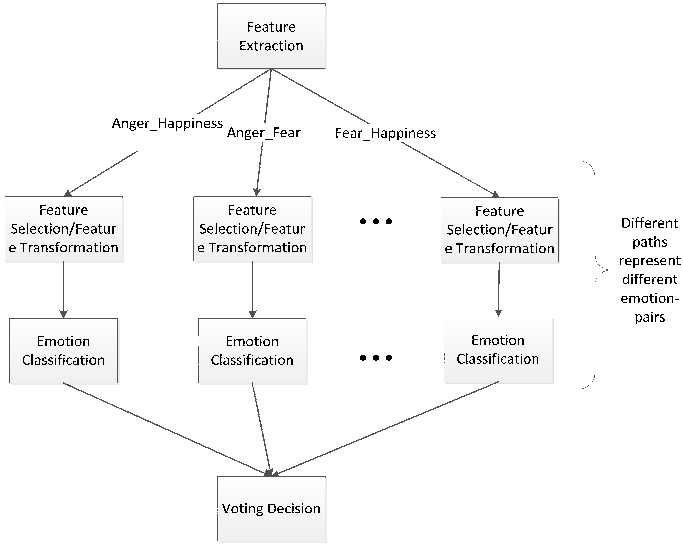}}
\caption{Flow chart of bi-classification and voting.}
\label{fig:flow}
\end{figure}

\subsection{Feature Extraction}
\label{ssec:feature_extraction}

The used acoustic features include the following low-level descriptors (LLDs): Intensity, Loudness, 12 MFCC, Pitch (F0), Probability of voicing, F0 envelope, 8 LSF (Line Spectral Frequencies), Zero-Crossing Rate. Delta regression coefficients are computed from these LLDs, and the following statistical functionals are applied to the LLDs and delta coefficients: Max./Min. value and respective relative position within input, range, arithmetic mean, 2 linear regression coefficients and linear and quadratic error, standard deviation, skewness, kurtosis, quartile 1-3, and 3 inter-quartile ranges. All the features are utterance-level features. In feature selection stage, the relevant feature subset or relevant feature space will be derived from the above large feature set.

\subsection{Feature Subspace Selection}
\label{ssec:feature_subspace_selection}

Different from traditional methods that distinguish all emotions with just one global feature subspace, this work selects different feature subspaces for different combination of emotion-pairs. For a specific emotion-pair, its corresponding feature subspace should be of the best power in distinguishing the two emotions of the pair. In order to verify the generality of our idea, the methods of feature subset selection and feature space transformation has been considered. Genetic algorithm (GA) is used for feature subset selection, while neural network (NN) is used for feature space transformation.

GA is a kind of stochastic searching and optimizing algorithm, that simulates the natural evolution process. We use a fixed number of features to form a vector (so called individual), and a fixed number of individuals to form the first population. Crossover and mutation operation is then used to generate a new individual. New population will be selected by comparing fitness. The ``Wrapper'' method is used to calculate the fitness of individuals, i.e. the accuracy of the classifier is used as the fitness. The above procedures are repeated until the average fitness of population reaches the threshold or the evolutionary generation reaches the threshold. Compared to some other heuristic searching algorithm, such as Sequential Floating Forward Selection (SFFS), it is more flexible to control the computing time for GA, especially when the feature set is relatively large.

\subsection{Emotion Classification}
\label{ssec:emotion_classification}

By using the feature subspace obtained in the previous step, a particular classifier  can be trained for a specific emotion-pair and be designated to distinguish the emotions in that emotion-pair. As each classifier is only related to a specific emotion-pair, we call it bi-classifier. For feature subset selection, two basic classification algorithms are used, including Logistic Regression (LR) and Support Vector Machine (SVM). For feature space transformation, neural network (NN) is used as the classifier.

\subsection{Voting Decision}
\label{ssec:voting_decision}

After getting the emotion distinguishing result for each emotion-pair in the previous emotion classification step, a voting and competition method is finally used to integrate the emotion classification results for all emotion-pairs to derive the final emotion recognition result. The voting decision process is summarized in Algorithm~\ref{alg:voting}.

\begin{algorithm}[!htb]
    \caption{Voting Decision Algorithm}
    \begin{algorithmic}[1]
        \REQUIRE ~~
            \\
            Input:\\
            $M$: the number of emotions\\
            $E=\{e_i|i=1,2,...,M\}$: emotion set\\
            $R=\{r_{e_ie_j}|e_i \neq e_j; r_{e_ie_j}, e_i, e_j \in E\}$: classification result of bi-classifier\\

        \ENSURE ~~
        \STATE Compute the number of different emotions in $R$: $N_e=\{n_{e_i}|e_i \in R\}$\\
        \STATE Create an emotion set with the maximum number in $N_e$: $E_{max}=\{m_k|m_k \in \mathop{\argmax}_{n_{e_k}}{E}; k=1,2,...,K\}$\\
        \STATE $e_m:=m_1$
        \IF {$K=1$}
            \RETURN $e_m$
        \ELSE
            \FOR{$k=2$ to K}
                \STATE $e_m:=r_{e_mm_k}$
            \ENDFOR
            \RETURN $e_m$
        \ENDIF
    \end{algorithmic}
\label{alg:voting}
\end{algorithm}

It can be proved that the final emotion recognition result can be correctly derived by the above voting decision algorithm given that all the bi-classifiers give the correct distinguishing result for each emotion-pair. The theorem and proof procedure is described in Theorem~\ref{thm:voting_proof}.

\newtheorem{theorem}{Theorem}
\begin{theorem}
Voting decision will be able to derive the correct result given all bi-classifiers are in correct situation.
\label{thm:voting_proof}
\end{theorem}

\begin{proof}
    Given the symbol definitions in Algorithm~\ref{alg:voting}, let $e_i$ be the target emotion and $e_i \in E$. So,
    \begin{align*}
        R{\quad}is{\quad}correct&\Rightarrow n_{e_i}=M-1\\
        &\Rightarrow n_{e_j}<M-1, e_j \neq e_i\\
        &\Rightarrow e_m=e_i
    \end{align*}
\end{proof}

\section{Experiment}
\label{sec:expeiment}
\subsection{Experimental Setting}
\label{ssec:experimental setting}

In this study, we used the well known Berlin emotional database (EmoDB)~\cite{burkhardt2005}. Ten actors (5 male and 5 female) simulated the emotions, producing 10 German sentences (5 short and 5 longer). EmoDB comprises 535 utterances that cover 6 archetypal emotions and
1 neutral emotion from everyday communication, namely, Anger, Fear, Happiness, Sadness, Disgust, Boredom and Neutral. Our work focuses on speaker independent emotion recognition, hence the samples of 8 actors (4 males and 4 females) are used as the training set, and the samples from the other 2 actors (1 male and 1 female) are used as the test set. The 5-fold-cross-validation method is used to conduct the experiments. OpenSmile toolkit~\cite{eyben2013} is used to extract acoustic features, and a total of 988 features are obtained.

Two experiment are conducted. In the first one, GA is used to select feature subset for each emotion-pair. As for emotion classification, the same emotion classifier is used for all emotion-pairs, but trained with different features subsets associated with different emotion-pairs. Furthermore, the same classifier is also used to recognize the emotions from the feature subsets associated with all emotions. This experiment is to verify that selecting the feature subset associated with emotion-pairs is better than the feature subset associated with all emotions using the same classifier. The details of parameter setting for GA are as follows: individual size 50, population size 100, two-point crossover with crossover probability 0.8, substitution mutation with mutation probability 0.1. If the generation number reaches 300 or the fitness value does not improve for the last 100 generation, the GA algorithm stops. In this experiment, 50 most representative features are selected by the GA algorithm to form the feature subset for not only every emotion-pair but also all emotions. Furthermore, two different classifiers (LR and SVM) are tested.

In the second experiment, neural network (NN) is used not only for feature space transformation but also as the classifier. This experiment is to verify that feature space transformation to the feature space associated with emotion-pairs is better than the feature space associated with all emotions using the same feature space transformation method. For experimental settings, the neural network has a 988-unit input layer corresponding to the dimensionality of original feature vector, and one 50-unit hidden layer corresponding to the dimensionality of feature subset in the feature selection method. Batch gradient descend method is used to learn the weights and the activation function is the sigmoid function. The learning rate is set to 0.1.

\subsection{Experimental Result}
\label{ssec:experimental result}
\subsubsection{Feature Selection Method}
\label{sssec:feature_selection_method}

We first conduct the feature selection experiment by comparing the similarity degree of the feature subspace for each emotion-pair and the global feature subspace for all emotions. Figure~\ref{fig:fs_compare} depicts the number of the common features (vertical axis) that are shared between the feature subspace for a specific emotion-pair and the global feature subspace for all emotions, where the horizontal axis represents different emotion-pairs (e.g. N-A is the Neutral-Anger pair). It should be noted that all feature subspaces (including emotion-pair specific feature subspace, and the global feature subspace for all emotions) contain 50 selected features, as described in Section~\ref{ssec:experimental setting}. From the figure, it can be seen that the number of the common features between each emotion-pair and all emotions is no more than 5 (5 out of 50). This indicates that the feature subspace of each emotion-pair is quite different from the global feature space of all emotions. This further confirms the necessity to perform pair-wised emotion classification with feature subspace related to emotion-pairs.

\begin{figure}[!htb]
  \centering
  \centerline{\includegraphics[width=9cm]{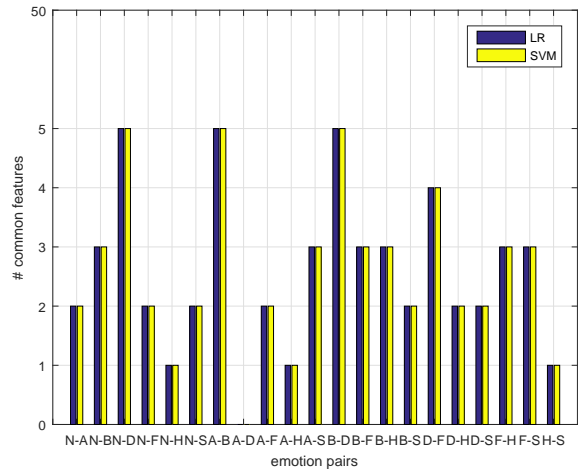}}
\caption{The number of the common features (vertical axis) shared between the feature subspace for each emotion-pair and the global feature subspace for all emotions. (N:Neutral, A:Anger, B:Boredom, H:Happiness, S:Sadness, D:Disgust, F:Fear)}
\label{fig:fs_compare}
\end{figure}

We further conducted emotion recognition experiment to compare the emotion recognition accuracies of different feature selection criterions between the proposed method and the traditional method. Experimental results are shown in Table~\ref{tab:accuracy_fs}, where ``Bi-classification and voting'' is the proposed method, while ``Multi-classification'' is the traditional method using the global feature subspace for all emotions. As can be seen, the recognition accuracy obtained by ``Bi-classification and voting'' is significantly higher than that using the ``Multi-classification'' method ($P<0.05$ by T-test).

\begin{table}[!htb]
\caption{Comparison of emotion recognition accuracy by using different feature selection criterions. ($P<0.05$)}
\begin{tabular}{c|c|c}
    \hline
    & Logistic Regression & SVM\\
    \hline
    Bi-classification and voting & 0.735 & 0.625 \\
    \hline
    Multi-classification & 0.653 & 0.513 \\
    \hline
\end{tabular}
\footnotesize
\label{tab:accuracy_fs}
\end{table}

The emotion recognition accuracy (recognition rate) of different emotions are further computed and shown in Figure~\ref{fig:fs_rate},  where \textbf{Bi-clf and voting} represents ``Bi-classification and voting'', \textbf{Multi-clf} represents ``Multi-classification''. It should be noted that the result about Disgust is not depicted because there are only quite few utterances with Disgust emotion. The experimental results indicate that the ``Bi-classification and voting'' method achieve better performance than the ``Multi-classification'' method for all emotions using both classifiers (LR and SVM). This result further proves that the priori information of the emotion-pair is helpful to feature selection and can bring further performance improvements for emotion recognition.

\begin{figure}[!htb]
  \centering
  \centerline{\includegraphics[width=9cm]{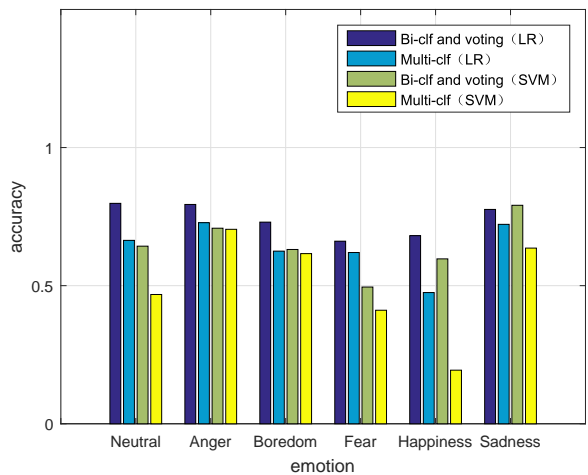}}
\caption{Emotion recognition accuracy of different emotions by using different feature selection criterions.}
\label{fig:fs_rate}
\end{figure}

\subsubsection{Feature Space Transformation}
\label{sssec:feature_space_transformation}

Similarly, we also conduct experiments in the feature space transformation scenario to validate the efficiency of our proposed divide and conquer idea for emotion recognition. The emotion recognition accuracy (recognition rate) by using different classification criterions withe feature space transformation are shown in Table~\ref{tab:accuracy_ft}, and the recognition of different emotions are showed in Figure~\ref{fig:ft_rate}.

\begin{table}[!htb]
\caption{Comparison of emotion recognition accuracy by using feature space transformation.($P<0.05$)}
\begin{tabular}{c|c}
    \hline
    & Neural Network\\
    \hline
    Bi-classification and voting & 0.652 \\
    \hline
    Multi-classification & 0.552 \\
    \hline
\end{tabular}
\footnotesize
\label{tab:accuracy_ft}
\end{table}

\begin{figure}[!htb]
  \centering
  \centerline{\includegraphics[width=9cm]{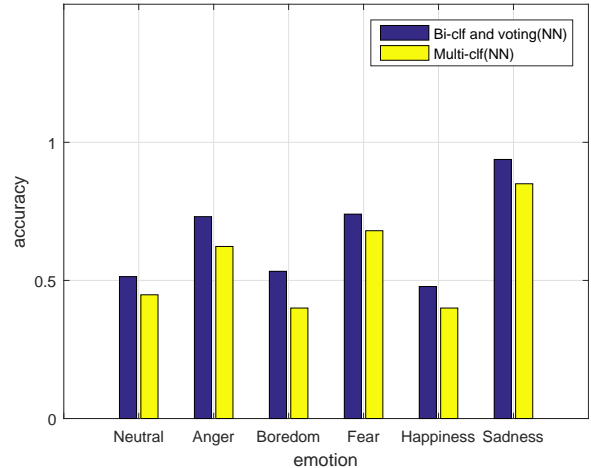}}
\caption{Emotion recognition accuracy of different emotions for different classification methods by using feature space transformation.}
\label{fig:ft_rate}
\end{figure}

From the experimental results, we can get the same conclusion as the feature selection method. It is indicated that the method of ``Bi-classification and voting'' is also effective in feature space transformation. This confirms the generality of our proposed method. For details of the experiment, please refer to our code and document on GitHub\footnote{$git@github.com:mxmaxi007/Emotion\_Recognition.git$}.

\section{Conclusion}
\label{sec:conclusion}

In this paper we present a ``Bi-classification and voting'' method by distinguishing different emotion-pairs in different feature space. The experimental results have proved that this method can get better result compared to the traditional multi-classification method. In addition, our method is a kind of divide and conquer algorithm which converts a complex multi-classification problem into many simple bi-classification problems. This idea makes it possible to boost the multi-class emotion recognition performance by optimizing the emotion classification performance for each emotion-pair. Hence, our future work will be devoted to the classifier optimization of different emotion-pairs.

\section{Acknowledgement}
\label{sec:acknowledgement}

This work is supported by National High Technology Research and Development Program of China (2015AA016305), National Natural Science Foundation of China (NSFC) (61375027, 61433018 and 61370023), joint fund of NSFC-RGC (Research Grant Council of Hong Kong) (61531166002, N\_CUHK404/15) and Major Program for National Social Science Foundation of China (13\&ZD189).


\bibliographystyle{IEEEbib}
\bibliography{refs}

\begin{thebibliography}{10}

\bibitem{ayadi2011}
M.S.~Kamel M.~El~Ayadi and F.~Karray,
\newblock ``Survey on speech emotion recognition: features, classification
  schemes and database,''
\newblock {\em Pattern Recognition}, vol. 44, no. 3, pp. 572--587, 2011.

\bibitem{dash1997}
M.~Dash and H.~Liu,
\newblock ``Feature selection for classification,''
\newblock {\em Intelligent Data Analysis}, vol. 1, no. 3, pp. 131--156, 1997.

\bibitem{guyon2003}
I.~Guyon and A.~Elisseeff,
\newblock ``An introduction to variable and feature selection,''
\newblock {\em Journal of machine learning research}, , no. 3, pp. 1157--1182,
  2003.

\bibitem{belhumeur1997}
J.P.~Hespanha P.N.~Belhumeur and D.J. Kriegman,
\newblock ``Eigenfaces vs. fisherfaces: Recognition using class specific linear
  projection,''
\newblock {\em IEEE Transactions on Pattern Analysis and Machine Intelligence},
  vol. 19, no. 7, pp. 711--720, 1997.

\bibitem{busso2009}
S.~Lee C.~Busso and S.~Narayanan,
\newblock ``Analysis of emotionally salient aspects of fundamental frequency
  for emotion detection,''
\newblock {\em IEEE Transactions on Audio, Speech, and Language Processing},
  vol. 17, no. 4, pp. 582--596, 2009.

\bibitem{cowie2001}
E.~Douglas-Cowie R.~Cowie and N.~Tsapatsoulis,
\newblock ``Emotion recognition in human-computer interaction,''
\newblock {\em IEEE Signal Processing Magazine}, vol. 18, no. 1, pp. 32--80,
  2001.

\bibitem{vayrynen2013}
J.~Kortelainen E.~Vayrynen and T.~Seppanen,
\newblock ``Classifier-based learning of nonlinear feature manifold for
  visualization of emotional speech prosody,''
\newblock {\em IEEE Transactions on Affective Computing}, vol. 4, no. 1, pp.
  47--56, 2013.

\bibitem{rabiner1978}
L.R. Rabiner and R.W. Schafer,
\newblock {\em Digital Processing of Speech Signals},
\newblock Prentice Hall, Upper Saddle River, New Jersey 07458, USA, 1978.

\bibitem{atal1974}
B.S. Atal,
\newblock ``Effectiveness of linear prediction characteristics of the speech
  wave for automatic speaker identification and verification,''
\newblock {\em the Journal of the Acoustical Society of America}, vol. 55, no.
  6, pp. 1304--1312, 1974.

\bibitem{davis1980}
S.~Davis and P.~Mermelstein,
\newblock ``Effectiveness of linear prediction characteristics of the speech
  wave for automatic speaker identification and verification comparison of
  parametric representations for monosyllabic word recognition in continuously
  spoken sentences,''
\newblock {\em IEEE Transactions on Acoustics, Speech, and Signal Processing},
  vol. 28, no. 4, pp. 357--366, 1980.

\bibitem{gobl2003}
C.~Gobl and A.~N. Chasaide,
\newblock ``The role of voice quality in communicating emotion, mood and
  attitude,''
\newblock {\em Speech Communication}, vol. 40, no. 1-2, pp. 189--212, 2003.

\bibitem{ververidis2008}
D.~Ververidis and C.~Kotropoulos,
\newblock ``Fast and accurate sequential floating forward feature selection
  with the bayes classifier applied to speech emotion recognition,''
\newblock {\em Signal Processing}, vol. 88, no. 12, pp. 2956--2970, 2008.

\bibitem{ferri1993}
V.~Kadirkamanathan F.~J.~Ferri and J.~Kittler,
\newblock ``Feature subset search using genetic algorithms,''
\newblock in {\em in: IEE/IEEE Workshop on Natural Algorithms in Signal
  Processing, IEE}. 1993, Press.

\bibitem{yu2013}
M.L.~Seltzer D.~Yu and J.~Li,
\newblock ``Feature learning in deep neural networks - studies on speech
  recognition tasks,''
\newblock {\em arXiv:1301.3605}, 2013.

\bibitem{schlosberg1954}
H.~Schlosberg,
\newblock ``Three dimensions of emotion,''
\newblock {\em Psychological Review}, vol. 61, no. 2, pp. 81--88, 1954.

\bibitem{ortony1990}
A.~Ortony and T.J. Turner,
\newblock ``What's basic about basic emotions,''
\newblock {\em Psychological Review}, vol. 97, no. 3, pp. 315--331, 1990.

\bibitem{lugger2008}
M.~Lugger and B.~Yang,
\newblock {\em Psychological Motivated Multi-Stage Emotion Classification
  Exploiting Voice Quality Features},
\newblock Speech Recognition, France Mihelic and Janez Zibert (Ed.), InTech,
  DOI: 10.5772/6383., 2008.

\bibitem{burkhardt2005}
A.~Paeschke F.~Burkhardt and M.~Rolfes,
\newblock ``A database of german emotional speech,''
\newblock in {\em Proceedings Interspeech 2005}, 2005, pp. 1517--1520.

\bibitem{eyben2013}
F.~Weninger F.~Eyben and F.~Gross,
\newblock ``Recent developments in opensmile, the munich open-source multimedia
  feature extractor,''
\newblock in {\em Proceedings of the 21st ACM international conference on
  Multimedia}. 2013, pp. 835--838, ACM New York, ISBN: 978-1-4503-2404-5
  DOI:10.1145/2502081.2502224.

\end{thebibliography}

\end{document}